\documentclass{article}

\def\ParSkip{} 
\usepackage{amssymb,amsmath,amsthm,bbm}
\usepackage{verbatim,float,url,dsfont}
\usepackage{graphicx}
\usepackage{mathtools,enumitem}
\usepackage{multirow}
\usepackage{ragged2e}
\usepackage{array}

\usepackage[colorlinks=true,citecolor=blue,urlcolor=blue,linkcolor=blue]{hyperref}
\usepackage[margin=1in]{geometry}
\usepackage[round]{natbib}

\usepackage[utf8]{inputenc} 
\usepackage[T1]{fontenc}    
\usepackage{booktabs}       
\usepackage{nicefrac}         
\usepackage{microtype}      
\usepackage{amsfonts}
\usepackage{authblk}
\usepackage{algorithm2e}
\usepackage{caption}
\usepackage{subcaption}

\SetKwInOut{Parameter}{Parameters}

\ifdefined\TimesFont 
\usepackage{times} 
\fi

\ifdefined\ParSkip 
\usepackage{parskip} 
\fi

\newtheorem{theorem}{Theorem}

\theoremstyle{definition}
\newtheorem{remark}{Remark}

\newtheorem*{assumption*}{\assumptionnumber}
\providecommand{\assumptionnumber}{}


\makeatletter
\newcommand*\rel@kern[1]{\kern#1\dimexpr\macc@kerna}
\newcommand*\widebar[1]{%
  \begingroup
  \def\mathaccent##1##2{%
    \rel@kern{0.8}%
    \overline{\rel@kern{-0.8}\macc@nucleus\rel@kern{0.2}}%
    \rel@kern{-0.2}%
  }%
  \macc@depth\@ne
  \let\math@bgroup\@empty \let\math@egroup\macc@set@skewchar
  \mathsurround\z@ \frozen@everymath{\mathgroup\macc@group\relax}%
  \macc@set@skewchar\relax
  \let\mathaccentV\macc@nested@a
  \macc@nested@a\relax111{#1}%
  \endgroup
}
\makeatother

\def\cL{\mathcal{L}}

\def\cX{\mathcal{X}}
\def\cY{\mathcal{Y}}

\def\cC{\mathcal{C}}

\title{Conformal Risk Control for Ordinal Classification}

\author[1]{Yunpeng Xu}
\author[2]{Wenge Guo\thanks{Author e-mail addresses: yx8@njit.edu, wenge.guo@njit.edu, zhi.wei@njit.edu}}
\author[1]{Zhi Wei}
\affil[1]{Department of Computer Science\\
       New Jersey Institute of Technology}
\affil[2]{Department of Mathematical Sciences\\
       New Jersey Institute of Technology}
       
\date{\today}

\begin{document}
\maketitle

\begin{abstract}
\end{abstract}
As a natural extension to the standard conformal prediction method, several conformal risk control methods have been recently developed and applied to various learning problems. In this work, we seek to control the conformal risk in expectation for ordinal classification tasks, which have broad applications to many real problems. For this purpose, we firstly formulated the ordinal classification task in the conformal risk control framework, and provided theoretic risk bounds of the risk control method. Then we proposed two types of loss functions specially designed for ordinal classification tasks, and developed corresponding algorithms to determine the prediction set for each case to control their risks at a desired level. We demonstrated the effectiveness of our proposed methods, and analyzed the difference between the two types of risks on three different datasets, including a simulated dataset, the UTKFace dataset and  the diabetic retinopathy detection
dataset.

\section{Introduction}\label{sec:intro}
In many decision making settings, a black box machine learning system is no longer adequate. Instead, we expect our system to not only make the predictions but also to quantify uncertainties and to control their risks \citep{hüllermeier2021}. This is especially true in certain high risk areas such as medical diagnosis and automatic driving. 

One solution to the problem is conformal prediction \citep{Vovk1999}, which has gained a lot of attention recently due to its many advantages. It is distribution-free, rigorous in statistics, and is easy to integrate with many machine learning models. The goal of conformal prediction is to create uncertainty sets for predictions made by these models, so that a certain coverage or risk requirement can be satisfied. 

To extend the notion of error by conformal prediction, recently a new framework called conformal risk control has been developed \citep{Angelopoulos2022}. Compared with a traditional conformal prediction method, conformal risk control generalizes the miscoverage rate to any bounded non-increasing loss functions, which offers a lot of flexibility to problems where other metrics are valued over the miscoverage rate. Yet, it still remains a question how to develop proper loss functions and to derive the corresponding prediction set for specific problems. 

In this paper, we focus on the ordinal classification \citep{mccullagh1980}, which is widely applied to many real life problems. In these problems, there exists a relative ordering among different classes in their label space. This ordinal nature brings unique challenges to the measurement of the prediction errors. For example, in a task of computer-aided medical diagnosis (CAMD) \citep{yanase2019}, it is much more harmful to mis-diagnose a severe condition than a mild condition, which indicates that different weights to different classes should be considered. While in a task of predicting a company's revenue range, it is desired that the predicted revenue range is as close as possible to the actual range, which indicates that the distance between the actual range and the prediction set should be considered. Both can be captured by the conformal risk control framework. 

For this purpose, we develop the conformal risk control method specifically for the ordinal classification problems. Our goal is to construct proper prediction sets in the ordinal setting so that their expected loss of prediction can be controlled. Our major contributions are four-fold.
\begin{itemize}
    \itemsep0em 
    \item Formulated the ordinal classification problem in the risk control framework, along with three conditions for an ideal prediction set in the ordinal setting. 
    \item Provided the upper and lower bounds of risk for the proposed risk control method. 
    \item Proposed two different types of risk for constructing prediction sets in the ordinal classification setting, and developed corresponding algorithms to find the optimal prediction sets. 
    \item Demonstrated the effectiveness of the method on both simulated and real data, and compared the difference between the two types of risks. 
\end{itemize}
In addition, in the Supplementary Materials, we present a general result on the lower bound of the conformal risk, which is a generalization of Theorem 2 in \citep{Angelopoulos2022}. Except that it is used to prove Theorem 1 in this paper, this result might be of independent interest. 

\subsection{Related work}
Conformal prediction was firstly developed by Vladimir Vovk and collaborators \citep{Vovk1999, vovk2005}. \citet{shafer2008}, \citet{Angelopoulos2021}, and \citet{fontana2023} provide a good introduction and survey to this field of work. 
Our work is primarily based on split conformal prediction \citep{papadopoulos2002, lei2015}. Other types of conformal prediction methods have been developed, such as cross-conformal prediction \citep{vovk2015, vovk2018} and CV+/Jackknife+ \citep{barber2021, kim2020}. Recently conformal prediction methods have been extended to the settings of non-exchangeability \citep{tibshirani2019, cauchois2020, podkopaev2021, gibbs2021, barber2022},
and have improved conditional coverage \citep{vovk2012, lei2014, barber2021, bian2022}. 

Our work is most relevant to risk control in expectation \citep{Angelopoulos2022} and ordinal conformal prediction \citep{Lu2022}. The former introduced a general framework of conformal risk control in expectation and discussed possible applications to several problems such as tumor segmentation, multi-label classification, hierarchical image classification and question answering. The latter developed a general method of ordinal conformal prediction sets with guaranteed marginal coverage for rating the disease severity in medical image. In this paper, we use the framework of conformal risk control in expectation, with the focus on ordinal classification problems.  
 
Other relevant work includes PAC-type risk control and its applications \citep{bates2021, angelopoulos2021b, park2019, park2021, angelopoulos2022c, schuster2022} and conformal prediction methods for conventional binary or multi-class classification \citep{lei2014b, hechtlinger2018, sadinle2019, romano2020, cauchois2021, angelopoulos2020, kuchibhotla2023}.

\section{Method}\label{sec:method}
\subsection{Problem Formulation}
Consider an ordinal classification problem with $K$ classes, i.e., $\mathcal{Y} = \{0, ... , K-1\}$, where there exists an ordering among these classes. Let $X_{test}$ be an input data and $Y_{test}\in \mathcal{Y}$ be its corresponding ground truth label. In the context of conformal prediction, since the goal is to produce a set of possible predictions that satisfy a required confidence level, let $C(X_{test}) \subseteq \{1, ..., K\}$ be the prediction set for $X_{test}$ to quantify the uncertainty associated with the model's predictions. Assume we also have a calibration dataset $\{(X_i, Y_i)\}_{i=1}^n$ that are drawn exchangeably from the same unknown distribution $P_{XY}$.

Let $L(Y_{test}, C(X_{test}))$ be a certain loss function defined as
\begin{equation}\label{eq1}
\begin{split}
    L(Y_{test}, C(X_{test}))= g(Y_{test}, C(X_{test}))\mathbf{1}(Y_{test}\notin C(X_{test}))
\end{split}
\end{equation}
where $g$ is a weight function, and it measures the loss if the true label $Y_{test}$ does not fall into the prediction set $C(X_{test})$; its value decreases as the prediction set $C(X_{test})$ grows. 

The choice of the weight function may impact the resulting prediction sets. In an actual ordinal classification problem, different classes may have different importance and large prediction errors are often more concerned. Accordingly, in our work, we propose two forms of weight functions: a weight-based loss function and a divergence-based loss function. The former assigns different weights to different classes, while the latter incurs a loss proportional to the divergence between the true label and the prediction set.

Note that for a standard split conformal prediction, it's hard to satisfy these requirements, which motivated us to introduce the new set-based loss functions. By leveraging the conformal risk control framework, it's easy to extend the loss function from an indicator type of loss function to a broad spectrum of loss functions, which suits the ordinal classification problem well.

Our goal is to construct $C(X_{test})$ so that the expected loss can be controlled at a specified $\alpha$ level, i.e., 
\begin{equation}\label{eq2}
    E[L(Y_{test}, C(X_{test})] \leq \alpha.
\end{equation}

Aside from the risk control by (\ref{eq2}), we argue that an ideal prediction set $C(X_{test})$ in an ordinal setting should also satisfy the following three conditions:
\begin{enumerate}
    \itemsep0em 
    \item[C1.] \label{condition1} $C(X_{test})$ should be a contiguous range of classes on $\mathcal{Y}$, i.e., $C(X_{test}) = [l, u]$, where $l\leq u \in \mathcal{Y}$, and the interval $[l, u]$ stands for $\{l,l+1,..., u-1, u\}$. This condition is needed because there is no real ordinal problem where a prediction set covers the neighboring classes while the one in the middle is skipped. 
    \item[C2.] \label{condition2} The prediction set should cover the point prediction $\hat{y} \in \mathcal{Y}$, i.e., $\hat{y} \in C(X_{test})$. Here $\hat{y}$ is the class assignment when the classification model predicts a single label. This condition is needed to ensure that a conformal prediction set provides a consistent result with a regular point prediction.
    \item[C3.] \label{condition3} The prediction sets should be nested for different levels of $\alpha$, i.e., $C_1(X_{test}) \subseteq C_2(X_{test})$ if $\alpha_1 \geq \alpha_2$, where $C_i$ is the prediction set corresponding  to level $\alpha_i$. This is a technical assumption but it aligns well with our intuition that as the risk level increases, the prediction set should gradually reduce in a smooth way. 
\end{enumerate}

\subsection{Risk control}

To derive predictions sets for different level of risk, we can simply split the task as two steps: (1) construct a sequence of sets $C_\lambda$ indexed by a threshold $\lambda \in \Lambda$; (2) determine the appropriate value $\hat \lambda$ of $\lambda$ such that the risk of the corresponding prediction set $C_{\hat \lambda}$ is controlled at a desired risk level $\alpha$. In this subsection, we focus on step 2 of this task by developing a general approach for determining $\hat \lambda$ with proven risk control of $C_{\hat \lambda}$, whereas in the subsequent subsections 2.3 and 2.4, we develop optimal algorithms for constructing a sequence of prediction sets $C_\lambda$ specifically tailed to two different types of loss functions arising in the ordinal classification. 

Specifically, consider that $C_\lambda = [l(\lambda), u(\lambda)]$ and the loss function $L(\lambda)$ is defined as in (\ref{eq1}). Suppose that $l(\lambda)$ is decreasing and $u(\lambda)$ is increasing in $\lambda$, both right-continuous, and there exists a value $\lambda_0 \in \Lambda$ such that $[l(\lambda_0), u(\lambda_0)] = \mathcal{Y}$. We also suppose that the weight function $g(Y, C)$ in (\ref{eq1}) is decreasing in $C$ in the sense that if $C_1 \subseteq C_2 \subseteq \mathcal{Y}$, we have $g(Y, C_1) \ge g(Y, C_2)$. Based on these assumptions, it's easy to check that the $C_\lambda$ and $L(\lambda)$ specifically defined for the ordinal classification satisfy the following properties: (i) $L(\lambda)$ is non-increasing in $\lambda$, right-continuous; (ii) $\inf_{\lambda} L(\lambda) = 0$ and $\sup_{\lambda} L(\lambda) \le B < \infty$ almost surely.

Let $L_i(\lambda) = L(Y_i, C_\lambda(X_i))$ be the loss value of the calibration observation $(X_i, Y_i)$ for $i = 1, \ldots, n$. For any desired risk level upper bound $\alpha$, pick the value $\hat \lambda$ of $\lambda$ below so that the risk of $C_{\hat \lambda}$ is controlled:
\begin{equation} \label{eq7}
\hat{\lambda}=\inf \{\lambda: \sum_{i=1}^nL_i(\lambda)\leq(n+1)\alpha-B \}.							    
\end{equation}

\begin{theorem} \label{theorem1}
Let $\hat{\lambda}$ be the value defined in equation (\ref{eq7}), under the assumptions stated as above, we have
$E[L(Y_{test}, C_{\hat{\lambda}}(X_{test})] \leq \alpha$. Specifically, in the settings of Theorem 4 in the Supplementary 
Materials, we have $\alpha - \frac{(M+2)B}{n+1} \leq E[L(Y_{test}, C_{\hat{\lambda}}(X_{test})]$, where $M$ is a non-negative integer given in the discontinuity assumption of Theorem 4. 
\end{theorem}

\begin{proof}
The first part of Theorem 1 directly follows from Theorem 1 in \citep{Angelopoulos2022} and 
the second part follows from Theorem 4 in the Supplementary Materials, which is a generalization 
of Theorem 2 in \citep{Angelopoulos2022}. \end{proof}

\begin{remark}
The risk control of the algorithm $C_{\hat{\lambda}}$ is established by using Theorem 1 of \citep{Angelopoulos2022}, however, its lower bound cannot be automatically obtained by using Theorem 2 therein, since that result is based on the  discontinuity assumption, which is often not satisfied in the the ordinal classification setting. As one of the reviewers pointed out, a trick of randomization or interpolation to transform $L_i(\lambda)$ to be a continuous function can be used to handle the issue. Instead of using such a tweak, the discontinuity assumption is relaxed
here such that it is often satisfied in the ordinal classification setting. Then, under the relaxed assumption, Theorem 2 of \citep{Angelopoulos2022} is generalized as Theorem 4 in the Supplementary Materials, which is proved by using the similar arguments as in that paper. Although this result is weaker than Theorem 2 of \citep{Angelopoulos2022}, it is often applicable in our ordinal classification setting. This result itself might be of independent interest to other applications.
\end{remark}

\subsection{Weight-based Risk}
For the weight-based risk function, the weight function $g$ is independent of the choice of $C(X_{test})$, i.e., $g(y, C(X_{test}))=h(y)$. This type of functions are particularly suitable to the cases where we would intentionally adjust the importance of certain classes. One example of $h(y)$ is constant for any label $y$, for which the risk corresponds to the conventional mis-coverage rate. Another example of $h(y)$ is $h(y)=y$, which assigns a higher weight to a higher class. 

\subsubsection{An Oracle Method}
Suppose we have the oracle access to the true conditional probability distribution $f(i|x)$ of $Y_{test}=i$ given $X_{test}=x$, where $i \in \mathcal{Y}$. 
By using this weight function, the conditional risk given $X_{test}=x$ can be written as:
\begin{equation} \label{eq3}
\begin{split}
    \quad E[L(Y_{test}, C(X_{test}))|X_{test}=x] =\sum_{i}{h(i)f(i|x)}-\sum_{i\in [l, u]}{h(i)f(i|x)}.
\end{split}
\end{equation}
We denote the first term $\sum_{i}{h(i)f(i|x)}$ of (\ref{eq3}) by $D(x)$, which is the upper bound of the conditional risk given $X_{test} = x$, i.e.,
$D(x) \leq \sum_{i}{\max\{h(i)\}f(i|x)} = \max\{h(i)\}$. For the simplicity of expression, we normalize $h(y)$ so that $\max\{h(i)\}=1$, in which case $D(x)\leq 1$. Therefore, the range of the risk value is $[0,1]$.

Clearly, controlling the conditional risk in (\ref{eq3}) is equivalent to
\begin{equation} \label{eq4}
    \sum_{i\in [l, u]} {h(i)f(i|x)} \geq D(x)- \alpha, 		
\end{equation}
which suggests the following rule to derive the optimal prediction set $[l, u]$ for $X_{test}=x$,
\begin{equation}
\begin{split}
    (l, u) = \underset{0\leq l \leq u \leq K-1}{\mathrm{argmin}}\, \Big\{u-l: \sum_{i\in[l, u]}h(i)f(i|x)\geq D(x)-\alpha \Big\}.	  		    
\end{split}
\end{equation}

\subsubsection{Marginal Risk Control}
Due to $f(i|x)$ is unknown in practice, we therefore replace it by the classification model output as an estimate of the probability. Thus, we can define a sequence of prediction set indexed by a threshold value, i.e., 
\begin{equation} \label{eq6}
\begin{split}
(l(\lambda), u(\lambda)) = \underset{0\leq l \leq \hat{y} \leq u \leq K-1}{\mathrm{argmin}}\, \Big\{u - l: 1-\sum_{i\in[l, u]}h(i)\hat{f}(i|x) \leq \lambda\Big\},	     
\end{split}
\end{equation}
where $\hat{f}(i|x)$ is the model score assigned to $x$ on class $i$ and $\hat{y} = \mathrm{argmax}_i\{h(i)\hat{f}(i|x)\}$ is the point prediction value.
\begin{remark}
We seek to control the marginal risk in this work, which is weaker than the conditional risk control, since it holds only on average \citep{vovk2012}, \citep{lei2014}, and \citep{barber2021}.
\end{remark}
\begin{remark}
If $h(i)$ is constant, then (\ref{eq6}) will reduce to the one proposed by \citet{Lu2022}. In other words, Lu et al's method is a special case of the weight-based risk control. 
\end{remark}

Under the conditions C1--C3, we solve the problem of (\ref{eq6}) as proposed by Algorithm \ref{alg1}. Basically, the algorithm works in a greedy way and runs at a complexity of $O(K)$. It starts with the point prediction label to grow the prediction set step by step until the risk requirement is satisfied, therefore it guarantees that the point prediction $\hat{y}$ is always covered within the prediction set. It also ensures that the nested property holds since it does not shrink the prediction set in this process. 

\RestyleAlgo{ruled}
\SetKwComment{Comment}{/* }{ */}

\begin{algorithm}
\DontPrintSemicolon 
\KwIn{$\lambda$, $h(i)$, $\hat{f}(x|i)$ for $i \in \{0, ..., K-1\}$}
\KwOut{$l, u$}
$s(i) \gets h(i)\hat{f}(x|i)$ for $i \in \{0, ..., K-1\}$\\
$u, l \gets \mathrm{argmax}_i\{s(i)\}$\\
$sum \gets 0$ \\
\While{$sum < (1-\lambda$)} { \label{alg1:terminatecondition}
  \If{$s(l-1) > s(u+1)$} {
    $sum = sum + s(l-1)$, $l = l-1$
  }
  \Else {
    $sum = sum + s(u+1)$, $u = u+1$
  }
}
\Return{$l, u$}\;
\caption{Determine the prediction set for a given $\lambda$}  \label{alg1}
\end{algorithm}


\begin{theorem} \label{theoremoptimal1}
The prediction set derived by Algorithm \ref{alg1} satisfies conditions C1--C3 and is optimal in the sense of  satisfying (\ref{eq6}). 
\end{theorem}
\begin{proof}
Let $C_{\lambda} = [l(\lambda), u(\lambda)]$ denote the prediction set derived by Algorithm \ref{alg1} and
$C^*_{\lambda}$ be the optimal prediction set satisfying  (\ref{eq6}) and conditions C1--C3. It's easy to check that $C_{\lambda}$ satisfies conditions C1--C3. In the following, we prove by contradiction that $C_{\lambda}$ also satisfies (\ref{eq6}). Assume that there exists a value $\lambda_0$  of $\lambda$  such that $C_{\lambda_0} \neq C^*_{\lambda_0}$. Let $\lambda_1 = \sup \{\lambda: C_{\lambda} = C^*_{\lambda}\}$ and $C_{\lambda_1} = C^*_{\lambda_1} = [l_1, u_1]$. 

Without loss of generality, suppose $s(l_1-1) > s(u_1+1)$. It is easy to see that for $\lambda\in (\lambda_1, \lambda_1 + s(u_1+1)]$, $C_{\lambda} = [l_1-1, u_1]$ and $C^*_{\lambda} = [l_1, u_1+1]$, due to the definition of $\lambda_1$. However, for $\lambda\in (\lambda_1 + s(u_1+1), \lambda_1 + s(l_1-1)]$, $C_{\lambda} = [l_1-1, u_1]$ keeps unchanged, but $C^*_{\lambda} = [l_1, u_1+1]$ needs to grow by one class so that it can satisfy (\ref{eq6}), which contradicts the optimality assumption of $C^*_{\lambda}$. 
\end{proof}
%

%

To calculate the actual value of $\lambda$, we can do a linear search. However, for an improved efficiency, we can also do a binary search as shown in Algorithm \ref{alg2}.
\begin{algorithm} 
\DontPrintSemicolon 
\KwIn{\{$x_i$, $y_i$\} for $i \in \{0, ..., n\}$}
\KwOut{$\hat{\lambda}$}
\Parameter{precision $\delta$}
$\lambda_0 \gets 0$, $\lambda_1 \gets 0.5$ \\
\While{$\Delta\lambda_k = |\lambda_k - \lambda_{k-1}| > \delta$} {
  $L(\lambda_k) \gets 0$ \\
  \For{$i\leftarrow 1$ \KwTo $n$}{
    Calculate $(l_i, u_i)$ for $x_i$ and $\lambda_k$ using Algorithm \ref{alg1} \\
    Calculate $L_i(\lambda_k)$ from $(l_i, u_i)$ and $y_i$\\
    $L(\lambda_k) \gets L(\lambda_k) + L_i(\lambda_k)$
  }
 \If{$L(\lambda_k) > (n+1)\alpha - 1$} {
 $\lambda_{k + 1} \gets \lambda_{k} - \frac{\Delta\lambda_k}{2}$
 }
 \Else {
 $\lambda_{k + 1} \gets \lambda_{k} + \frac{\Delta\lambda_k}{2}$
}
}
\Return{$\hat{\lambda} = \lambda_{k+1}$}\;
\caption{Determine the value of $\hat{\lambda}$ for a given $\alpha$} \label{alg2}
\end{algorithm}

\subsection{Divergence-based risk}
For the divergence based risk, the loss is proportional to the distance between the true label and the prediction set, i.e., $g(Y_{test}, C(X_{test}))=\inf\{d(y,i): i\in [l, u]\}$, where $d(\cdot,\cdot) \in \mathcal{Y}\times\mathcal{Y} \rightarrow \mathbf{R}^+$ is a given distance measure on the label space.  This type of functions are particularly suitable for the cases where we are more concerned with large prediction errors than the differences among the individual classes. 

There could be different ways to define the difference measure $d(\cdot,\cdot)$. For example, $d(y,i)=|y-i|^p$. In this work, we only consider the case where $p=1$. Furthermore, we normalize the loss value by $K-1$, therefore,
\begin{equation} \label{eq8}
g(Y_{test}, C(X_{test}))=\frac{1}{K-1}\inf\big\{|y-k|: k\in[l, u]\big\}. 						    
\end{equation}

Clearly, when $l=0$ and $u=K-1$, the prediction set covers the whole label space, we have $L(Y_{test}, C(X_{test}))=0$. The upper bound of the risk, on the other hand, is $1$, which happens when the true label and the prediction set are exactly on two opposite boundaries of the label space. Therefore $E[L(Y_{test}, C(X_{test}))|X_{test}]\in [0, 1]$.

\subsubsection{An Oracle Method}
Suppose we have the oracle access to the true conditional probability distribution $f(i|x)$ of $Y_{test}$ given $X_{test} = x$. Then the conditional risk for a given $X_{test}=x$ can be written as:
\begin{equation} \label{eq9}
\begin{split}
E\big[L(Y_{test}, C(X_{test}))|X_{test}=x\big] = \frac{1}{K-1}\Big(\sum_{i<l}(l-i)f(i|x)+\sum_{i>u}(i-u)f(i|x)\Big).   
\end{split}
\end{equation}

It is easy to see that the first term of (\ref{eq9}) increases as $l$ increases, while the second term decreases as $u$ increases. The optimal prediction set $[l^*, u^*]$ for $X_{test}=x$ is therefore given by
\begin{equation} \label{eq10}
\begin{split}
(l^*, u^*)=\underset{0\leq l \leq u \leq K-1}{\mathrm{argmin}}\Big\{u-l: \sum_{i<l}(l-i)f(i|x)
+\sum_{i>u}(i-u)f(i|x)\leq(K-1)\alpha\Big\}.        
\end{split}
\end{equation}
 
\subsubsection{Marginal Risk Control}
Similarly, we use the classification model output as an estimate of $f(i|x)$ and seek to control the marginal risk instead. Under this setting, we define a sequence of prediction set $[l(\lambda), u(\lambda)]$ indexed by a threshold value $\lambda$, i.e.,
\begin{equation} \label{eq11}
\begin{split}
(l(\lambda), u(\lambda)) =\underset{0\leq l \leq \hat{y} \leq u \leq K-1}{\mathrm{argmin}} \Big\{u-l:  \sum_{i<l}(l-i)\hat{f}(i|x) + \sum_{i>u}(i-u)\hat{f}(i|x) \leq \lambda\Big\},	       
\end{split}
\end{equation}
where $\hat{f}(i|x)$ is the model score assigned to $x$ on class $i$, and  $\hat{y} = \mathrm{argmax}_i\{\hat{f}(i|x)\}$ is the point prediction value.

To determine the prediction set for a given $\lambda$, it helps to firstly look into how risk changes when the prediction set size adjusts by $1$. For the simplicity of expression, let $R(l, u)$ be the risk incurred by a prediction set $C(X_{test)}=[l, u]$, which can be calculated by (\ref{eq9}). It's easy to see that
\begin{equation}
\begin{split}
    R(l+1, u)-R(l, u)=\sum_{i<l+1}(l+1-i)\hat{f}(i|x) 
    -\sum_{i<l}(l-i)\hat{f(i|x)}=\sum_{i\leq l}\hat{f}(i|x),
\end{split}
\end{equation}
Similarly, 
\begin{equation}
\begin{split}
R(l, u-1)-R( l, u)=\sum_{i>u-1}(i-(u-1))\hat{f}(i|x)
-\sum_{i>u}(i-u)\hat{f}(i|x)=\sum_{i\geq u}\hat{f}(i|x),
\end{split}
\end{equation}
Therefore, the risk change incurred by every single adjustment of the prediction set boundary can be easily calculated beforehand at a complexity of $O(K)$. With this insight, we propose Algorithm \ref{alg3} below to determine the prediction set under conditions C1--C3. 

\begin{algorithm}
\DontPrintSemicolon 
\KwIn{$\lambda$, $\hat{f}(x|i)$ for $i \in \{0, ..., K-1\}$}
\KwOut{$l, u$}
Calculate $head(j)=\sum_{i\leq j}f(i|x)$, $tail(j)=\sum_{i\geq j}f(i|x)$ for $i\in \{0, ..., K-1\}$\\
$l, u \gets  \mathrm{argmax}_i\{\hat{f}(x|i)\}$\\ 
$sum \gets R(l, u)$ \\

\While{$sum > \lambda$} {
  \If{$head(l-1) \geq tail(u+1)$} {
    $sum = sum - head(l-1)$ \\
    $l = l -1$
  }
  \Else {
    $sum = sum - tail(u+1) $ \\
    $u = u+1$
  }
}
\Return{$l, u$}\;
\caption{Determine the prediction set for a given $\lambda$}  \label{alg3}
\end{algorithm}

Basically, this is a greedy algorithm that runs at a complexity of $O(K)$. The algorithm starts with the class that has the maximum value of $\hat{f}(i|x)$, therefore it guarantees that the point prediction $\hat{y}$ is always covered within the prediction set. The algorithm grows the prediction set one class a step. At each step, it adjusts the prediction boundary by maximizing the risk reduction of that step, until it reaches the set that satisfies the risk requirement. Since it does not shrink the prediction set in this process, the nest property is also ensured.

To calculate the value of $\hat{\lambda}$ for a given $\alpha$, we follow the same binary search process as used by Algorithm \ref{alg2}. 

\begin{theorem} \label{theoremoptimal2}
The prediction set derived by Algorithm \ref{alg3} satisfies conditions C1--C3 and is optimal in the sense of  satisfying (\ref{eq11}). 
\end{theorem}
The proof is similar to the one for Theorem \ref{theoremoptimal1}, therefore is omitted here.

\section{Experiments}\label{sec:exp}
To demonstrate the effectiveness of the proposed algorithms, we evaluated them on three ordinal classification tasks, including a task to classify a simulated ordinal data, a task to predict the age group of the person based on their face images, and a task to predict the severity of patients' diabetic retinopathy based on their retina images\footnote{Codes for this work can be found at our GitHub repository: \href{https://github.com/yx8njit/ordinal-conformal-risk-control}{https://github.com/yx8njit/ordinal-conformal-risk-control}.}.

For the baseline method, we consider the method proposed by \citet{Lu2022}, which can be regarded as a special case of the weight-based risk, in which all classes have equal weights. To our knowledge, it is the only existing method that deals with conformal prediction for ordinal problems. 

For all experiments in this work, data are equally split for validation and for test. The softmax output of the ordinal classification model is fed as the input to our algorithm, and the results are averaged over 100 random trials.

\subsection{Simulated Data}
We simulated a 10-class ordinal data on a 2-D plate. Each class has 2,000 data points sampled from a Gaussian distribution, where the $i^{th}$ class is centered at the coordination $[i, i]$,  with a randomly generated covariance matrix. Figure \ref{fig:simulation_data} displays the distribution of the data. 

\begin{figure}[!htb]
  \centering
  \includegraphics[width=0.75\linewidth]{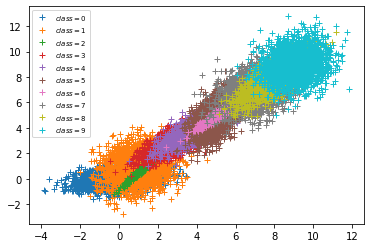}
  \caption{A simulated 10-class ordinal data}\label{fig:simulation_data}
\end{figure}

To classify the ordinal labels, we built a two-layer MLP with 50 neurons on the hidden layer. We use 14,000 data points for validation and test.

We evaluate the following 4 scenarios:
\begin{enumerate}
    \itemsep0em 
     \item [S1:] weight-based risk with equal weights on all classes.
    \item [S2:] weight-based risk with incremental weights on higher classes, where class $i$ has a weight $i$;
    \item [S3:] weight-based risk, with double weights on class 5$\sim$9;
    \item [S4:] divergence-based risk.
\end{enumerate}

For all scenarios, we evaluate the algorithm with different $\alpha$ values. The actual risks calculated are given in Table \ref{tab:sim-risk}, which shows that the proposed algorithm produces risk values very close to the $\alpha$ values for the first three scenarios. 

\begin{table}
    \centering
    \caption{Actual risk at different value of $\alpha$ on simulated data} \label{tab:sim-risk}
    \begin{tabular}{rllll}
      \toprule 
      \bfseries $\alpha$ & 0.02 & 0.08 & 0.14 & 0.20\\
      \midrule 
      S1 & 0.0199 & 0.0807 & 0.1403 & 0.1999\\
      S2 & 0.0201 & 0.0802 & 0.1402 & 0.2001\\
      S3 & 0.0201 & 0.0782 & 0.1406 & 0.1997\\
      S4 & 0.0199 & 0.0463 & 0.0463 & 0.0464\\
      \bottomrule 
    \end{tabular}
\end{table}

For the last scenario, the actual risk saturates around $0.0463$ since $\alpha=0.08$. This is due to the fact that the max divergence risk on this data set has been reached. Afterwards, the prediction set shrinks to a single class, therefore, the calculated risk no longer changes. 

Need to point out that although the divergence-based risk has the same value range of $[0, 1]$ as the weight-based risk, it generally has a much smaller max risk value. In fact, the risk value of 1 can only be reached if every data point in the dataset has a true label that is on the opposite side of the predicted set on the label space, which is impossible in practice. 

Figure \ref{fig:simulation_setsizes} illustrates the trend of prediction set sizes at different values of $\alpha$ for these scenarios. It can be seen that as the value of $\alpha$ increases, the set sizes reduce. For the divergence-based risk, the set size shrinks to 1 around $\alpha=0.05$, which aligns with the calculated max risk value $0.0463$. The figure also shows that, for the weight-based risk, both the incremental weights and the double weights scenarios have larger prediction sets compared with the equal weights scenario on the same $\alpha$ value. This is because the algorithm needs to extend the prediction set accordingly to reduce the overall risk. Therefore, there is a trade-off between the prediction set size and the risk reduction. 

\begin{figure}[!htb]
  \centering
  \includegraphics[width=0.75\linewidth]{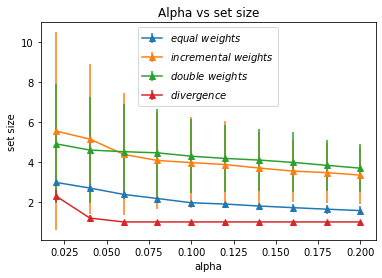}
  \caption{Prediction set sizes at different values of $\alpha$ on simulated data}\label{fig:simulation_setsizes}
\end{figure}

Figure \ref{fig:simulation-histograms} illustrates the risk distribution at a fixed $\alpha$ for different scenarios, where $\alpha$ is fixed at $0.10$ for the first three scenarios and at $0.02$ for the last scenario. It shows that for all the scenarios, the algorithm controls the risk well within narrow ranges around the specified $\alpha$ values, which are shown by the orange dotted lines in the graphs.  

\begin{figure}[ht]
\centering
\begin{subfigure}{.235\textwidth}
  \centering
  \includegraphics[width=1.05\linewidth]{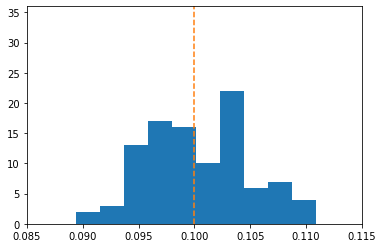}  
  \caption{equal weights\\ $\alpha=0.1$}
\end{subfigure}
\begin{subfigure}{.235\textwidth}
  \centering
  \includegraphics[width=1.05\linewidth]{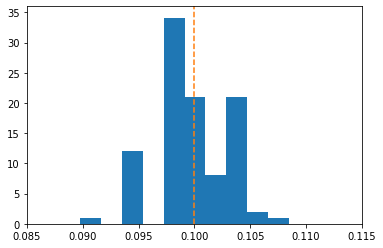}  
  \caption{incremental weights\\ $\alpha=0.1$}
\end{subfigure}
\begin{subfigure} {.235\textwidth}
  \centering
  \includegraphics[width=1.05\linewidth]{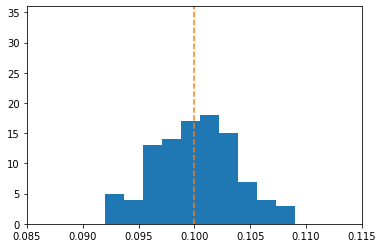}  
  \caption{double weights \\ $\alpha=0.1$}
\end{subfigure}
\begin{subfigure}{.235\textwidth}
  \centering
  \includegraphics[width=1.05\linewidth]{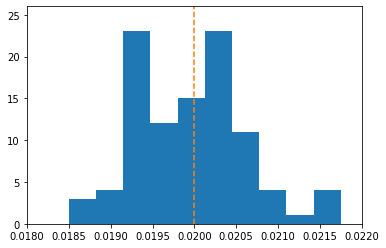}  
  \caption{divergence \\ $\alpha=0.02$}
\end{subfigure}
\caption{Risk distributions of different scenarios at a fixed $\alpha$ on the simulated dataset.} \label{fig:simulation-histograms}
\end{figure}

We also looked into the difference between the weight-based risk and the divergence-based risk. For this purpose, we compare their predicted class ranges. Due to their risk may take very different values, for a meaningful comparison, we carefully choose their $\alpha$ values so that their averaged prediction sets size are both 3. Then we compare the distribution of the centroids of their prediction sets, which are shown in Figure \ref{fig:simulation_centroids}. It can be seen that compared with the equal weights-based loss function, the divergence-based loss function tends to push the centroids toward the center of the label range. In other words, it is more centripetal than the weight-base risk. This can be explained by the fact that the divergence-based loss function punishes more on the extreme error cases where the true label lies on the opposite side of the prediction set.

\begin{figure}[!htb]
  \centering
  \includegraphics[width=0.75\linewidth]{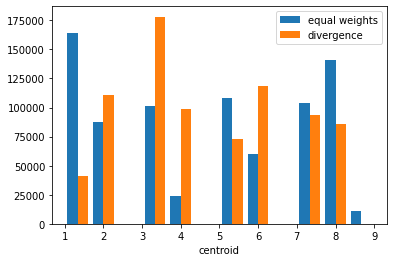}
  \caption{Distributions of prediction centroids on simulated data.}\label{fig:simulation_centroids}
\end{figure}

\subsection{Age Recognition}
For this task, we use the UTKFace dataset\footnote{The data is available at https://susanqq.github.io/UTKFace/.}, which is a large-scale face dataset with over 20K images that cover a long age span (range from 0 to 116 years old). Each image has been annotated of the person's age. In our work, we kept those below age 100, and discretized the age into 20 groups, where each group covers 5 years in range, i.e., group 0 is for $0\sim 4$, and group 1 is for $5\sim 9$ and so on. Here are some examples of the images and their group labels. 

\begin{figure}[ht]
\centering
\begin{subfigure}{.15\textwidth}
  \centering
  \includegraphics[width=0.7\linewidth]{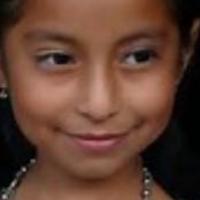}  
  \caption{age 8 \\ group 1}
\end{subfigure}
\begin{subfigure}{.15\textwidth}
  \centering
  \includegraphics[width=0.7\linewidth]{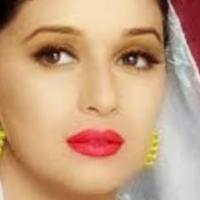}  
  \caption{age 25 \\ group 5}
\end{subfigure}
\begin{subfigure} {.15\textwidth}
  \centering
  \includegraphics[width=0.7\linewidth]{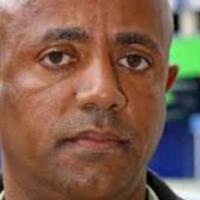}  
  \caption{age 42 \\ group 8}
\end{subfigure}
\begin{subfigure}{.15\textwidth}
  \centering
  \includegraphics[width=0.7\linewidth]{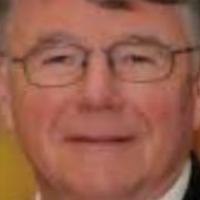}  
  \caption{age 67 \\ group 13}
\end{subfigure}
\begin{subfigure}{.15\textwidth}
  \centering
  \includegraphics[width=0.7\linewidth]{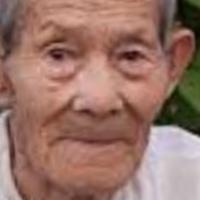}  
  \caption{age 86 \\ group 17}
\end{subfigure}
\caption{Image examples in the UTKFace dataset.} \label{fig:utkface-examples}
\end{figure}

We built an ordinal classifier to predict the age group using ResNet34 \citep{He2016}, where we didn’t particularly tune the hyper-parameters for a superb classification result. We use 17K images for validation and test, and evaluate the following scenarios:
\begin{itemize}
    \itemsep0em 
    \item S1: weighted risk, with equal weights on all classes
    \item S2: weighted risk, with double weights on class 0$\sim$3
    \item S3: weighted risk, with double weights on class 16$\sim$19
    \item S4: divergence risk
\end{itemize}

Here shows in Table \ref{tab:utkface-risk} the actual risks for different $\alpha$ values in these scenarios, which shows our algorithm produces the actual risks very close to the $\alpha$ value. It also shows that as $\alpha$ increases to certain values, the risks of the last two scenarios saturate to their max risk values. The trend of prediction set sizes over different $\alpha$ values and the risk distribution at a fixed $\alpha$ show similar conclusions as on the simulated data, and are therefore omitted here due to the length restriction. 
\begin{table}
    \centering
    \caption{Actual risks at different values of $\alpha$ on UTKFace}\label{tab:utkface-risk}
    \begin{tabular}{rllllll}
      \toprule 
      \bfseries $\alpha$ & 0.005 & 0.01 & 0.025 & 0.05 & 0.10 &0.20 \\
      \midrule 
      S1 & 0.005 & 0.010 & 0.025 & 0.050 & 0.100 & 0.201 \\
      S2 & 0.005 & 0.010 & 0.025 & 0.050 & 0.100 & 0.200 \\
      S3 & 0.005 & 0.010 & 0.025 & 0.050 & 0.100 & 0.169 \\
      S4 & 0.005 & 0.010 & 0.016 & 0.016 & 0.016 & 0.016 \\
      \bottomrule 
    \end{tabular}
\end{table}





We also fixed the prediction set size to be 3 and compared the distributions of the prediction set centroids for the different scenarios, which are shown in Figure \ref{fig:utkface_centroids}. From the distributions, we can see that by doubling the weights of the group $0\sim 3$, it pushes the distribution towards the lower age end, while by doubling the weights of the group $16\sim 19$, it pushes the distribution towards the higher age end. Similarly, the divergence loss function tends to be centripetal compared with other loss functions. 

\begin{figure}[!htb]
  \centering
  \includegraphics[width=0.75\linewidth]{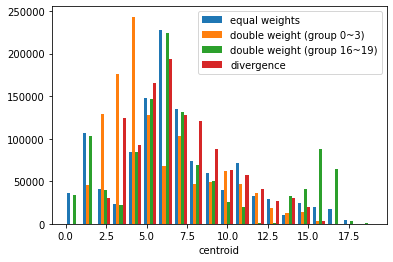}
  \caption{Distributions of prediction centroids on UTKFace}\label{fig:utkface_centroids}
\end{figure}

\subsection{Diabetic Retinopathy Detection}
For this task, we use the diabetic retinopathy detection dataset\footnote{The data is available at https://www.kaggle.com/\\competitions/diabetic-retinopathy-detection/data.}. It is a large set of over $35K$ retina images taken under a variety of imaging conditions. Each image has a clinician rating on the presence of diabetic retinopathy (DR) using a scale of 0 to 4, where 0 means no DR while 4 means a proliferative DR. Figure \ref{fig:dr-examples} shows some examples of the image along with their labels. 

\begin{figure}[ht]
\centering
\begin{subfigure}{.15\textwidth}
  \centering
  \includegraphics[width=0.9\linewidth]{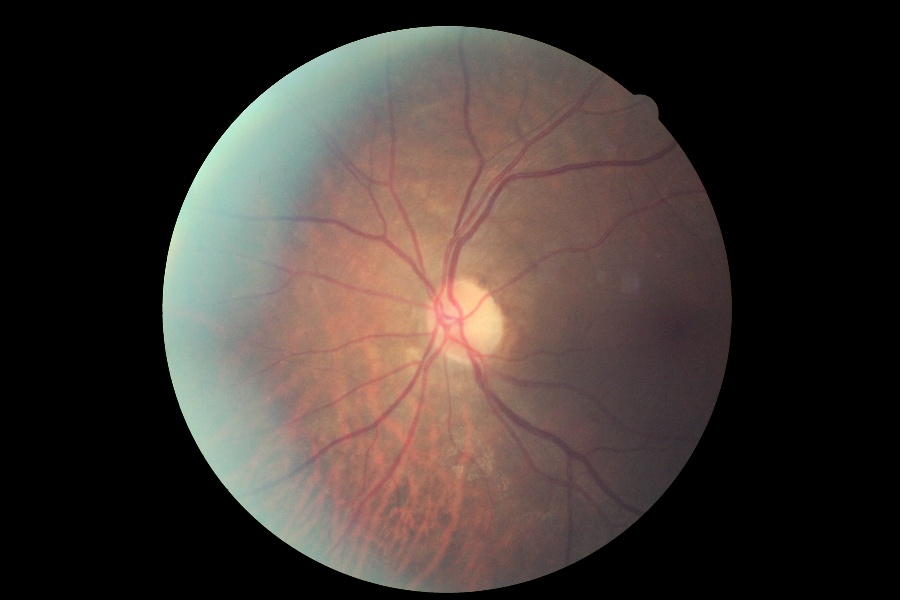}  
  \caption{0 - No DR}
\end{subfigure}
\begin{subfigure}{.15\textwidth}
  \centering
  \includegraphics[width=0.9\linewidth]{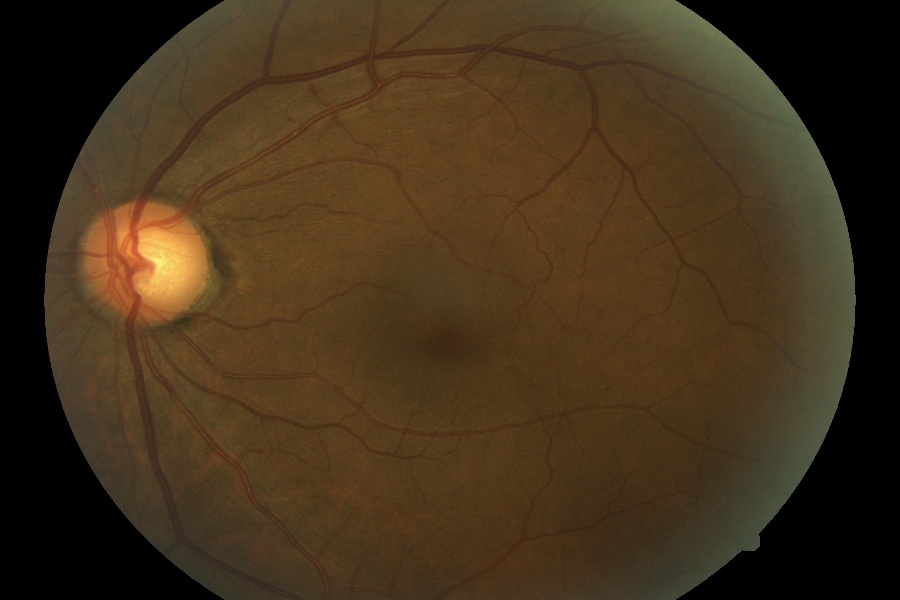}  
  \caption{1 - Mild}
\end{subfigure}
\begin{subfigure} {.15\textwidth}
  \centering
  \includegraphics[width=0.9\linewidth]{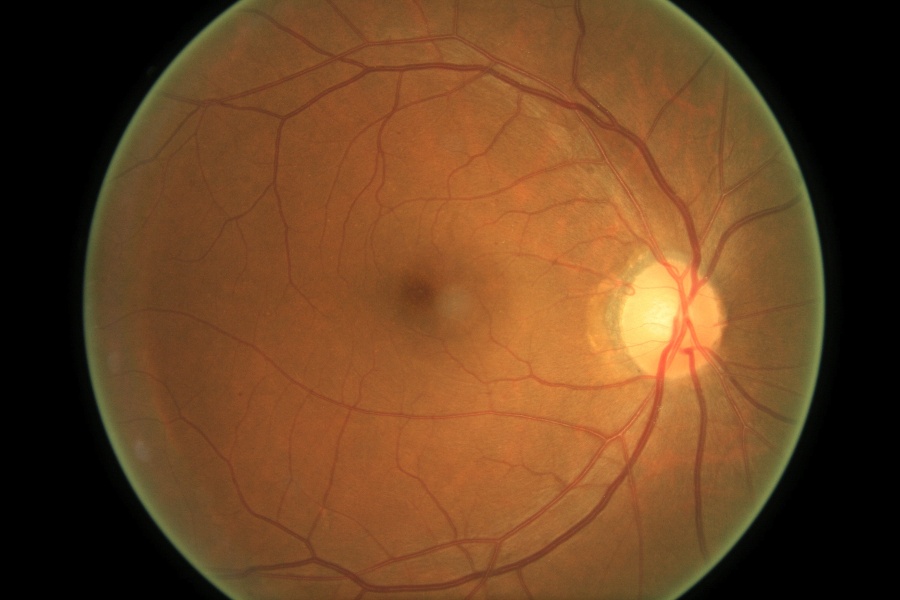}  
  \caption{2 - Moderate}
\end{subfigure}
\begin{subfigure}{.15\textwidth}
  \centering
  \includegraphics[width=0.9\linewidth]{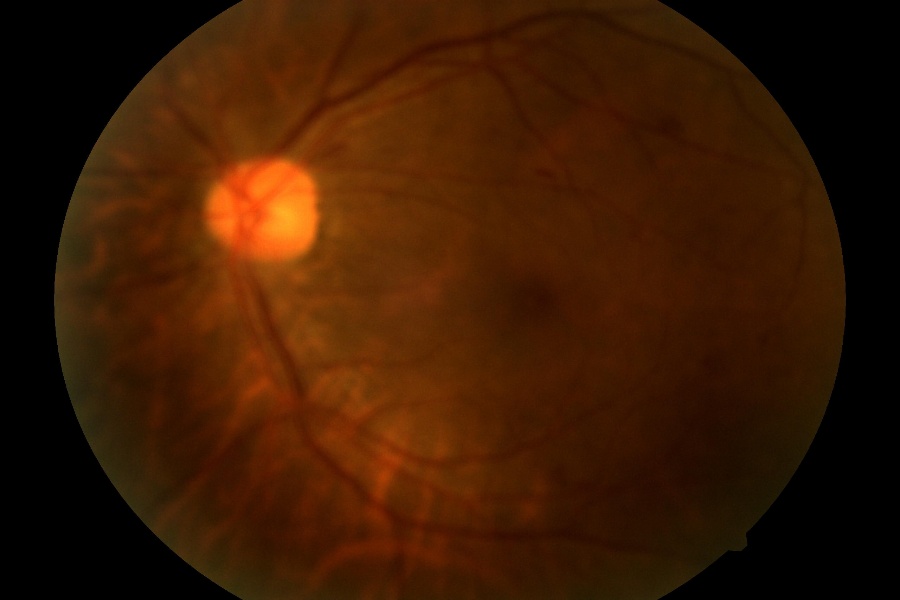}  
  \caption{3 - Severe}
\end{subfigure}
\begin{subfigure}{.15\textwidth}
  \centering
  \includegraphics[width=0.9\linewidth]{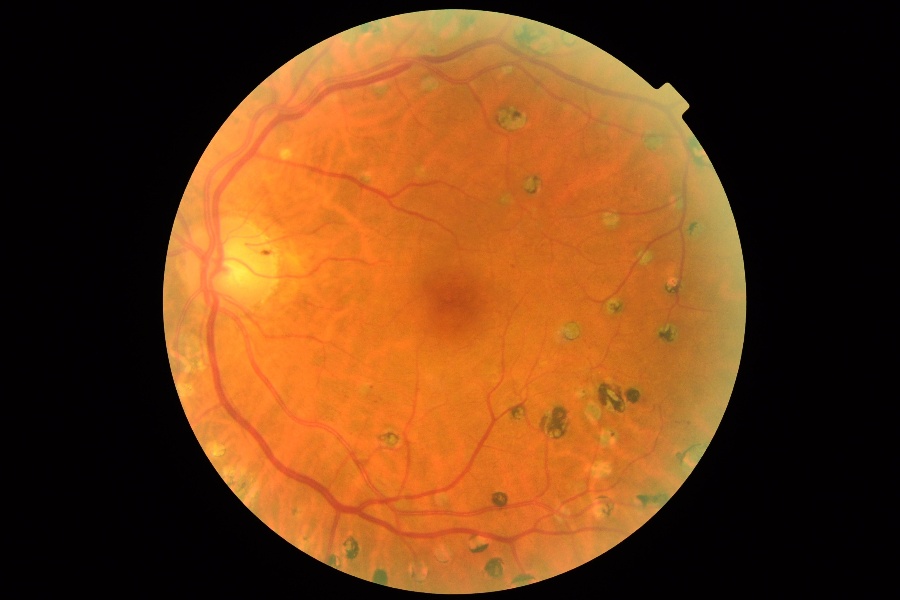}  
  \caption{4 - Prolifer.}
\end{subfigure}
\caption{Image examples in the DR dataset.} \label{fig:dr-examples}
\end{figure}

Similarly, we build an ordinal classifier to predict the age group using ResNet34, and do not particularly tune the hyper-parameters for a superb classification result. We use 19K images for validation and test, and evaluate the following scenarios:
\begin{itemize}
    \itemsep0em 
    \item S1: weighted risk, with equal weights on all classes
    \item S2: weighted risk, with double weights on class 3 \& 4
    \item S3: divergence risk
\end{itemize}

We show the expected risk for different alpha values of these scenarios in Table \ref{tab:dr-risk}. The trend of prediction set sizes over different $\alpha$ values and the risk distribution at a fixed $\alpha$ show similar conclusions as on the simulated data, and are therefore omitted here due to the length restriction. 

\begin{table}
    \centering
    \caption{Actual risk on different value of $\alpha$ on the DR dataset} \label{tab:dr-risk}
    \begin{tabular}{rllll}
      \toprule 
      \bfseries $\alpha$ & 0.02 & 0.08 & 0.14 & 0.20 \\
      \midrule 
      S1 & 0.0198 & 0.0794 & 0.1402 & 0.1499 \\
      S2 & 0.0199 & 0.0801 & 0.1401 & 0.1485 \\
      S3 & 0.0200 & 0.0430 & 0.0428 & 0.0428 \\
      \bottomrule 
    \end{tabular}
\end{table}
 


 We also fixed the prediction set size to be 2 and compared the distribution of the prediction set centroids for the different risk functions, as shown in Figure \ref{fig:dr_centroids}. Similarly, compared with the equal weights, doubling weights pushes the distribution slightly towards the higher end, while divergence risk pushes the distribution slightly towards the center.

 \begin{figure}[!htb]
  \centering
  \includegraphics[width=0.75\linewidth]{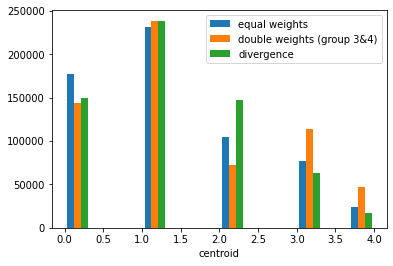}
  \caption{Distributions of prediction centroids on DR data}\label{fig:dr_centroids}
\end{figure}

\section{Conclusion \& Discussion}\label{sec:conclusion}
We formulated the ordinal classification task within the recently developed framework of conformal risk control in expectation, and introduced two types of loss functions specifically tailored to the learning task. Based on these two loss functions, we developed two conformal prediction algorithms, which are shown controlling the corresponding conformal risk at a desired level and are optimal in some sense. Simulation study and real data analysis showed effectiveness of the proposed algorithms. 

 With the proposed method, we can design appropriate weight functions accordingly to ensure the desired coverage for the specific task. Between the two types of the proposed loss functions, the weight-based risk is more suitable to the cases where we want to adjust the importance of certain classes, while the divergence-based risk is more suitable to the cases where we are more concerned with large prediction errors than the differences among the individual classes. By comparison, the divergence-based risk is more centripetal since it pushes the centroids of the prediction sets toward the center of the label range, while the weight-based risk pushes the centroids toward the classes where they have higher weights.

There are a few important questions remaining. Firstly, instead of the marginal coverage, how the choice of weight functions impacts the conditional coverage for each single class. Secondly, how to choose the weight function and the risk threshold that are fit to the specific problem. We will investigate these questions in our future work. 


\bibliographystyle{plainnat}
\bibliography{uai2023-template}

\begin{thebibliography}{38}
\providecommand{\natexlab}[1]{#1}
\providecommand{\url}[1]{\texttt{#1}}
\expandafter\ifx\csname urlstyle\endcsname\relax
  \providecommand{\doi}[1]{doi: #1}\else
  \providecommand{\doi}{doi: \begingroup \urlstyle{rm}\Url}\fi

\bibitem[Angelopoulos et~al.(2020)Angelopoulos, Bates, Malik, and
  Jordan]{angelopoulos2020}
Anastasios Angelopoulos, Stephen Bates, Jitendra Malik, and Michael~I Jordan.
\newblock Uncertainty sets for image classifiers using conformal prediction.
\newblock \emph{arXiv preprint arXiv:2009.14193}, 2020.

\bibitem[Angelopoulos and Bates(2021)]{Angelopoulos2021}
Anastasios~N. Angelopoulos and Stephen Bates.
\newblock {A gentle introduction to conformal prediction and distribution-free
  uncertainty quantification}.
\newblock \emph{arXiv:2107.07511}, 2021.

\bibitem[Angelopoulos et~al.(2021)Angelopoulos, Bates, Cand{\`e}s, Jordan, and
  Lei]{angelopoulos2021b}
Anastasios~N Angelopoulos, Stephen Bates, Emmanuel~J Cand{\`e}s, Michael~I
  Jordan, and Lihua Lei.
\newblock Learn then test: Calibrating predictive algorithms to achieve risk
  control.
\newblock \emph{arXiv preprint arXiv:2110.01052}, 2021.

\bibitem[Angelopoulos et~al.(2022{\natexlab{a}})Angelopoulos, Bates, Fisch,
  Lei, and Schuster]{Angelopoulos2022}
Anastasios~N. Angelopoulos, Stephen Bates, Adam Fisch, Lihua Lei, and Tal
  Schuster.
\newblock {Conformal risk control}.
\newblock \emph{arXiv:2208.02814}, 2022{\natexlab{a}}.

\bibitem[Angelopoulos et~al.(2022{\natexlab{b}})Angelopoulos, Krauth, Bates,
  Wang, and Jordan]{angelopoulos2022c}
Anastasios~N Angelopoulos, Karl Krauth, Stephen Bates, Yixin Wang, and
  Michael~I Jordan.
\newblock Recommendation systems with distribution-free reliability guarantees.
\newblock \emph{arXiv preprint arXiv:2207.01609}, 2022{\natexlab{b}}.

\bibitem[Barber et~al.(2021)Barber, Candes, Ramdas, and Tibshirani]{barber2021}
Rina~Foygel Barber, Emmanuel~J Candes, Aaditya Ramdas, and Ryan~J Tibshirani.
\newblock Predictive inference with the jackknife+.
\newblock \emph{Ann. Statist.}, 49:\penalty0 486--507, 2021.

\bibitem[Barber et~al.(2022)Barber, Candes, Ramdas, and Tibshirani]{barber2022}
Rina~Foygel Barber, Emmanuel~J Candes, Aaditya Ramdas, and Ryan~J Tibshirani.
\newblock Conformal prediction beyond exchangeability.
\newblock \emph{arXiv preprint arXiv:2202.13415}, 2022.

\bibitem[Bates et~al.(2021)Bates, Angelopoulos, Lei, Malik, and
  Jordan]{bates2021}
Stephen Bates, Anastasios Angelopoulos, Lihua Lei, Jitendra Malik, and Michael
  Jordan.
\newblock Distribution-free, risk-controlling prediction sets.
\newblock \emph{Journal of the ACM (JACM)}, 68\penalty0 (6):\penalty0 1--34,
  2021.

\bibitem[Bian and Barber(2022)]{bian2022}
Michael Bian and Rina~Foygel Barber.
\newblock Training-conditional coverage for distribution-free predictive
  inference.
\newblock \emph{arXiv preprint arXiv:2205.03647}, 2022.

\bibitem[Cauchois et~al.(2020)Cauchois, Gupta, Ali, and Duchi]{cauchois2020}
Maxime Cauchois, Suyash Gupta, Alnur Ali, and John~C Duchi.
\newblock Robust validation: Confident predictions even when distributions
  shift.
\newblock \emph{arXiv preprint arXiv:2008.04267}, 2020.

\bibitem[Cauchois et~al.(2021)Cauchois, Gupta, and Duchi]{cauchois2021}
Maxime Cauchois, Suyash Gupta, and John~C Duchi.
\newblock Knowing what you know: valid and validated confidence sets in
  multiclass and multilabel prediction.
\newblock \emph{The Journal of Machine Learning Research}, 22\penalty0
  (1):\penalty0 3681--3722, 2021.

\bibitem[Fontana et~al.(2023)Fontana, Zeni, and Vantini]{fontana2023}
Matteo Fontana, Gianluca Zeni, and Simone Vantini.
\newblock Conformal prediction: a unified review of theory and new challenges.
\newblock \emph{Bernoulli}, 29\penalty0 (1):\penalty0 1--23, 2023.

\bibitem[Gibbs and Candes(2021)]{gibbs2021}
Isaac Gibbs and Emmanuel Candes.
\newblock Adaptive conformal inference under distribution shift.
\newblock \emph{Advances in Neural Information Processing Systems},
  34:\penalty0 1660--1672, 2021.

\bibitem[He et~al.(2016)He, Zhang, Ren, and Sun]{He2016}
Kaiming He, Xiangyu Zhang, Shaoqing Ren, and Jian Sun.
\newblock {Deep residual learning for image recognition}.
\newblock \emph{In Proceedings of the 2016 IEEE Conference on Computer Vision
  and Pattern Recognition (CVPR)}, page 770–778, 2016.

\bibitem[Hechtlinger et~al.(2018)Hechtlinger, P{\'o}czos, and
  Wasserman]{hechtlinger2018}
Yotam Hechtlinger, Barnab{\'a}s P{\'o}czos, and Larry Wasserman.
\newblock Cautious deep learning.
\newblock \emph{arXiv preprint arXiv:1805.09460}, 2018.

\bibitem[Hüllermeier and Waegeman(2021)]{hüllermeier2021}
Eyke Hüllermeier and Willem Waegeman.
\newblock Aleatoric and epistemic uncertainty in machine learning: an
  introduction to concepts and methods.
\newblock \emph{Machine Learning}, 110:\penalty0 457–506, 2021.

\bibitem[Juri~Yanase(2019)]{yanase2019}
et~al Juri~Yanase.
\newblock A systematic survey of computer-aided diagnosis in medicine: Past and
  present developments.
\newblock \emph{Expert Systems with Applications}, 2019.

\bibitem[Kim et~al.(2020)Kim, Xu, and Barber]{kim2020}
Byol Kim, Chen Xu, and Rina Barber.
\newblock Predictive inference is free with the jackknife+-after-bootstrap.
\newblock \emph{Advances in Neural Information Processing Systems},
  33:\penalty0 4138--4149, 2020.

\bibitem[Kuchibhotla and Berk(2023)]{kuchibhotla2023}
Arun~K Kuchibhotla and Richard~A Berk.
\newblock Nested conformal prediction sets for classification with applications
  to probation data.
\newblock \emph{The Annals of Applied Statistics}, 17\penalty0 (1):\penalty0
  761--785, 2023.

\bibitem[Lei(2014)]{lei2014b}
Jing Lei.
\newblock Classification with confidence.
\newblock \emph{Biometrika}, 101\penalty0 (4):\penalty0 755--769, 2014.

\bibitem[Lei and Wasserman(2014)]{lei2014}
Jing Lei and Larry Wasserman.
\newblock Distribution-free prediction bands for non-parametric regression.
\newblock \emph{Journal of the Royal Statistical Society: Series B: Statistical
  Methodology}, pages 71--96, 2014.

\bibitem[Lei et~al.(2015)Lei, Rinaldo, and Wasserman]{lei2015}
Jing Lei, Alessandro Rinaldo, and Larry Wasserman.
\newblock A conformal prediction approach to explore functional data.
\newblock \emph{Annals of Mathematics and Artificial Intelligence},
  74:\penalty0 29--43, 2015.

\bibitem[Lu et~al.(2022)Lu, Angelopoulos, and Pomerantz]{Lu2022}
Charles Lu, Anastasios~N. Angelopoulos, and Stuart Pomerantz.
\newblock {Improving trustworthiness of AI disease severity rating in medical
  imaging with ordinal conformal prediction sets}.
\newblock \emph{arXiv:2207.02238}, 2022.

\bibitem[McCullagh(1980)]{mccullagh1980}
Peter McCullagh.
\newblock Regression models for ordinal data.
\newblock \emph{Journal of the Royal Statistical Society B}, 42:\penalty0
  109–142, 1980.

\bibitem[Papadopoulos et~al.(2002)Papadopoulos, Proedrou, Vovk, and
  Gammerman]{papadopoulos2002}
Harris Papadopoulos, Kostas Proedrou, Volodya Vovk, and Alex Gammerman.
\newblock Inductive confidence machines for regression.
\newblock In \emph{Machine Learning: ECML 2002: 13th European Conference on
  Machine Learning Helsinki, Finland, August 19--23, 2002 Proceedings 13},
  pages 345--356. Springer, 2002.

\bibitem[Park et~al.(2019)Park, Bastani, Matni, and Lee]{park2019}
Sangdon Park, Osbert Bastani, Nikolai Matni, and Insup Lee.
\newblock \uppercase{PAC} confidence sets for deep neural networks via
  calibrated prediction.
\newblock \emph{arXiv preprint arXiv:2001.00106}, 2019.

\bibitem[Park et~al.(2021)Park, Dobriban, Lee, and Bastani]{park2021}
Sangdon Park, Edgar Dobriban, Insup Lee, and Osbert Bastani.
\newblock \uppercase{PAC} prediction sets under covariate shift.
\newblock \emph{arXiv preprint arXiv:2106.09848}, 2021.

\bibitem[Podkopaev and Ramdas(2021)]{podkopaev2021}
Aleksandr Podkopaev and Aaditya Ramdas.
\newblock Distribution-free uncertainty quantification for classification under
  label shift.
\newblock In \emph{Uncertainty in Artificial Intelligence}, pages 844--853.
  PMLR, 2021.

\bibitem[Romano et~al.(2020)Romano, Sesia, and Candes]{romano2020}
Yaniv Romano, Matteo Sesia, and Emmanuel Candes.
\newblock Classification with valid and adaptive coverage.
\newblock \emph{Advances in Neural Information Processing Systems},
  33:\penalty0 3581--3591, 2020.

\bibitem[Sadinle et~al.(2019)Sadinle, Lei, and Wasserman]{sadinle2019}
Mauricio Sadinle, Jing Lei, and Larry Wasserman.
\newblock Least ambiguous set-valued classifiers with bounded error levels.
\newblock \emph{Journal of the American Statistical Association}, 114\penalty0
  (525):\penalty0 223--234, 2019.

\bibitem[Schuster et~al.(2022)Schuster, Fisch, Gupta, Dehghani, Bahri, Tran,
  Tay, and Metzler]{schuster2022}
Tal Schuster, Adam Fisch, Jai Gupta, Mostafa Dehghani, Dara Bahri, Vinh~Q Tran,
  Yi~Tay, and Donald Metzler.
\newblock Confident adaptive language modeling.
\newblock \emph{arXiv preprint arXiv:2207.07061}, 2022.

\bibitem[Shafer and Vovk(2008)]{shafer2008}
Glenn Shafer and Vladimir Vovk.
\newblock A tutorial on conformal prediction.
\newblock \emph{Journal of Machine Learning Research}, 9\penalty0 (3), 2008.

\bibitem[Tibshirani et~al.(2019)Tibshirani, Foygel~Barber, Candes, and
  Ramdas]{tibshirani2019}
Ryan~J Tibshirani, Rina Foygel~Barber, Emmanuel Candes, and Aaditya Ramdas.
\newblock Conformal prediction under covariate shift.
\newblock \emph{Advances in neural information processing systems}, 32, 2019.

\bibitem[Vovk(2012)]{vovk2012}
Vladimir Vovk.
\newblock Conditional validity of inductive conformal predictors.
\newblock In \emph{Asian conference on machine learning}, pages 475--490. PMLR,
  2012.

\bibitem[Vovk(2015)]{vovk2015}
Vladimir Vovk.
\newblock Cross-conformal predictors.
\newblock \emph{Annals of Mathematics and Artificial Intelligence},
  74:\penalty0 9--28, 2015.

\bibitem[Vovk et~al.(1999)Vovk, Gammerman, and Saunders]{Vovk1999}
Vladimir Vovk, Alex Gammerman, and Craig Saunders.
\newblock {Machine-learning applications of algorithmic randomness}.
\newblock \emph{Sixteenth International Conference on Machine Learning
  (ICML-1999)}, pages 444--453, 1999.

\bibitem[Vovk et~al.(2005)Vovk, Gammerman, and Shafer]{vovk2005}
Vladimir Vovk, Alexander Gammerman, and Glenn Shafer.
\newblock \emph{Algorithmic learning in a random world}, volume~29.
\newblock Springer, 2005.

\bibitem[Vovk et~al.(2018)Vovk, Nouretdinov, Manokhin, and Gammerman]{vovk2018}
Vladimir Vovk, Ilia Nouretdinov, Valery Manokhin, and Alexander Gammerman.
\newblock Cross-conformal predictive distributions.
\newblock In \emph{Conformal and Probabilistic Prediction and Applications},
  pages 37--51. PMLR, 2018.

\end{thebibliography}

\newpage

\appendix

\section{A generalization of Theorem 2 in \cite{Angelopoulos2022}}

In the Supplementary Material, we will present a generalization of Theorem 2 in \cite{Angelopoulos2022} on the lower bound of conformal risk control, which is needed in the proof of Theorem 1 in our paper. This result itself might be of independent interest to other applications. For convenience, we use the same notations as in \cite{Angelopoulos2022} in the following discussion. 

Suppose that $\cC_\lambda: \cX \rightarrow 2^{\cY}$ is a given sequence of functions of an input $X \in \cX$ that outputs a prediction set $\cC(X) \subseteq \cY$, which is indexed by a threshold $\lambda \in \Lambda$, and $\cL(Y, \cC_\lambda(X))  \in (-\infty, B]$ be a given loss function of any observation $(X, Y)$ and the corresponding prediction set $\cC_\lambda(X)$. For the calibration observations $(X_i, Y_i)_{i=1}^n$ and the test observation $(X_{n+1}, Y_{n+1})$, let $L_i(\lambda) = L(Y_i, C_{\lambda}(X_i))$ for $i=1, \ldots, n+1$ and $\hat R_{n}(\lambda) = (L_1(\lambda) + \ldots + L_n(\lambda))/n$. The value of $\lambda$ is determined according to the following algorithm: 
$$\hat \lambda = \inf \Big\{\lambda: \frac{n}{n+1} \hat R_{n}(\lambda) +  \frac{B}{n+1} \le \alpha \Big\},$$
where $\alpha \in (0, B)$ is the given desired risk level upper bound. Let $D = \{\lambda: J(\hat R_{n+1}, \lambda) > 0\}$ denote the set of discontinuities in $\hat R_{n+1}$, where $J(\cL, \lambda)$ is the jump function defined below,
$$J(\cL, \lambda) = \lim_{\epsilon \rightarrow 0+}\cL(\lambda-\epsilon) - \cL(\lambda),$$
which quantifies the size of the discontinuity in the loss function $\cL$ at a point $\lambda$. For any $\lambda \in D$, define $$s(\lambda) = |\{i: J(L_i, \lambda) >0\}|,$$
the number of $L_i(\lambda)$ which are discontinuous at $\lambda$. Regarding $s(\lambda)$, we assume that 
$$\sup_{\lambda \in \Lambda} s(\lambda) \le M, \quad \text{almost surely},$$ 
where $M$ is a non-negative integer. Specifically, if $M = 0$, this assumption implies that for any $\lambda$, $P(J(L_i, \lambda) >0) = 0$ for $i = 1, \ldots, n+1$, which is the exactly original  discontinuity assumption in Theorem 2 of \cite{Angelopoulos2022}. 

 Under the above relaxed assumption, we generalize Theorem 2 in \cite{Angelopoulos2022} as follows. This result is applicable to the ordinal classification setting. 

\textit{\textbf{Theorem 4. } In the settings of Theorem 1 of \cite{Angelopoulos2022}, further assume that $L_i$ are i.i.d, $L_i >0$ and 
$$\sup_{\lambda \in \Lambda} s(\lambda) \le M, \quad \text{almost surely},$$ 
where $M$ is a non-negative integer. Then,
$$E[L_{n+1}(\hat \lambda) ] \ge \alpha - \frac{(M+2)B}{n+1}.$$}

\noindent Specifically, if $M = 0$, the above result reduces to Theorem 2 in \cite{Angelopoulos2022}. To show Theorem 4, we need to generalize Lemma 1 of \cite{Angelopoulos2022} as follows and then use the similar arguments as in the proof of Theorem 2 therein along with this lemma. 

\textit{\textbf{Lemma 1.}} In the settings of Theorem 4, any jumps in the empirical risk are bounded, i.e., 
$$\sup_{\lambda \in \Lambda} J(\hat R_{n}, \lambda) \le \frac{(M+1)B}{n}, \quad \text{almost surely}.$$ 

This lemma can be proved by using the similar arguments as in the proof of Lemma 1 of \cite{Angelopoulos2022}. 

\end{document}